\documentclass[11pt]{article}

\usepackage[preprint]{acl}
\usepackage{amsthm}
\definecolor{lightgreen}{RGB}{230, 255, 230}
\definecolor{lightyellow}{RGB}{255, 255, 230}
\definecolor{lightred}{RGB}{255, 230, 230}
\definecolor{lightblue}{RGB}{230, 240, 255}
\usepackage[table]{xcolor}
\usepackage{booktabs, siunitx, xcolor}
\usepackage{makecell}
\usepackage{geometry} 
\usepackage{microtype} 
\usepackage{lipsum}    
\usepackage{enumitem}  
\usepackage{url}       
\usepackage{changes}   
\usepackage{xspace}    
\usepackage{float}

\usepackage{amsmath}    
\usepackage{amssymb}    
\usepackage{mathtools}  
\usepackage{calc}       

\usepackage{booktabs}       
\usepackage{tabularx}       
\usepackage{graphicx}       
\usepackage{caption}        
\usepackage{subcaption}     
\usepackage{pgfplots}       
\usepackage{pgfplotstable}  
\pgfplotsset{compat=1.17}   

\usepackage{tikz} 
\usepackage{pgfplots} 
\pgfplotsset{compat=1.18} 

\usepackage{algorithm}     
\usepackage{algorithmicx}  
\usepackage{algpseudocode} 
\usepackage{url}

\usepackage{amsmath, amssymb}
\usepackage{xcolor}
\usepackage{graphicx}

\usepackage{tikz}
\usetikzlibrary{arrows.meta,calc,patterns,positioning}
\usepackage{lmodern}
\usepackage{pgfplots}
\pgfplotsset{compat=1.18} 
\usepgfplotslibrary{polar} 
\usepackage{siunitx}       
\usepackage{booktabs}      
\usepackage[table]{xcolor} 

\newtheorem{lemma}{Lemma}[section] 

\usepackage{amsmath} 
\newcommand{\method}{\texttt{TEPO}\xspace}

\title{Token-Level Policy Optimization: Linking Group-Level Rewards to Token-Level Aggregation via Markov Likelihood}

\author{
  Xingyu Lin$^{1,2,3}$, Yilin Wen$^{1}$, En Wang$^{2,3}$, Du Su$^{4}$, \\
  Wenbin Liu$^{2,3}$, Chenfu Bao$^{1}$, Zhonghou Lv$^{1}$
  \thanks{Corresponding authors: enwang@jlu.edu.cn (En Wang), sudu@ict.ac.cn (Du Su), liuwenbin@jlu.edu.cn (Wenbin Liu), chenfubao@baidu.com (Chenfu Bao), zhonghoulv@baidu.com (Zhonghou Lv)} \\ 
  $^1$Baidu Inc; \\
  $^2$College of Computer Science and Technology, Jilin University; \\
  $^3$Key Laboratory of Symbolic Computation and Knowledge Engineering of MOE, Jilin University; \\
  $^4$State Key Laboratory of AI Safety, Institute of Computing Technology, Chinese Academy of Sciences \\
}

\begin{document}

\maketitle
\begin{abstract}
Group Relative Policy Optimization (GRPO) has significantly advanced the reasoning ability of large language models (LLMs), particularly by boosting their mathematical performance. However, GRPO and related entropy-regularization methods still face challenges rooted in the sparse token rewards inherent to chain-of-thought (CoT). Current approaches often rely on undifferentiated token-level entropy adjustments, which frequently lead to entropy collapse or model collapse.
In this work, we propose \method, a novel token-level framework that incorporates Markov Likelihood (sequence likelihood)  links group-level rewards with tokens via token-level aggregation. Experiments show that \method consistently outperforms existing baselines across key metrics (including @k and accuracy). It not only sets a new state of the art on mathematical reasoning tasks but also significantly enhances training stability.
\end{abstract}

\section{Introduction}
LLMs have significantly advanced mathematical reasoning capabilities by leveraging state-of-the-art reinforcement learning (RL) techniques. A pivotal advancement in this domain is GRPO, which not only enhances mathematical reasoning performance, but also addresses key limitations of prior methods like Reinforcement Learning from Human Feedback (RLHF) \cite{bai2022training}, such as memory efficiency, sample utilization, and training stability \cite{shao2024deepseekmath}. These advantages render GRPO indispensable for the extensive application of LLMs in mathematical reasoning tasks.

\begin{figure}[t]
\centering
\resizebox{\linewidth}{!}{
\begin{tikzpicture}
\begin{polaraxis}[
    width=0.95\linewidth,
    height=0.72\linewidth,
    ymin=0, ymax=0.8,
    ytick={0,0.2,0.4,0.6,0.8},
    yticklabel style={font=\scriptsize},
    grid=both,
    major grid style={dashed,gray!30},
    minor grid style={gray!15},
    xtick={0,45,90,135,180,225,270,315},
    xticklabels={AIME24, AIME25, AMC, MATH-500, OMNI-MATH, OlympiadBench, Minerva, Avg.},
    xticklabel style={font=\tiny, yshift=1pt},
    legend pos=south east,
    legend style={font=\tiny, draw=none, fill=white, fill opacity=0.7, 
                  text opacity=1, inner sep=2pt, outer sep=1pt,
                  cells={anchor=west}},
    legend columns=2,
    axis background/.style={fill=white},
    ylabel={Performance (Acc@k)},
    ylabel style={font=\bfseries\tiny, yshift=-3pt}
]

\addplot+[thick, mark=*, mark size=0.7pt, draw=blue!60!black, draw opacity=0.9,
          fill=blue!20, fill opacity=0.08, area legend]
coordinates {
    (0,0.1156) (45,0.0625) (90,0.4047) (135,0.7240)
    (180,0.2600) (225,0.3777) (270,0.3676) (315,0.3050)
} -- cycle;
\addlegendentry{w.\,GRPO}

\addplot+[thick, mark=*, mark size=0.7pt, draw=red!60, draw opacity=0.9,
          fill=red!20, fill opacity=0.08, area legend]
coordinates {
    (0,0.1125) (45,0.0500) (90,0.4280) (135,0.7660)
    (180,0.2500) (225,0.3792) (270,0.3308) (315,0.3085)
} -- cycle;
\addlegendentry{w.\,Clip-Higher}

\addplot+[thick, mark=*, mark size=0.7pt, draw=green!50!black, draw opacity=0.9,
          fill=green!20, fill opacity=0.08, area legend]
coordinates {
    (0,0.1083) (45,0.0740) (90,0.4284) (135,0.7560)
    (180,0.2611) (225,0.3934) (270,0.3786) (315,0.3163)
} -- cycle;
\addlegendentry{w.\,CLIP-Cov}

\addplot+[thick, mark=*, mark size=0.7pt, draw=purple!60, draw opacity=0.9,
          fill=purple!20, fill opacity=0.08, area legend]
coordinates {
    (0,0.1135) (45,0.0615) (90,0.4318) (135,0.7480)
    (180,0.2614) (225,0.4014) (270,0.3602) (315,0.3162)
} -- cycle;
\addlegendentry{w.\,Entropy-based Term}

\addplot+[ultra thick, dashed, dash pattern=on 4pt off 2pt, line cap=round,
          mark=*, mark size=1.0pt,
          draw=red!90!black, fill=red!40, fill opacity=0.15, area legend]
coordinates {
    (0,0.1218) (45,0.07604) (90,0.4356) (135,0.7720)
    (180,0.2657) (225,0.3748) (270,0.3566) (315,0.3204)
} -- cycle;
\addlegendentry{\textbf{w.\,ours (dashed)}}

\end{polaraxis}
\end{tikzpicture}
}
\vspace{-6pt}
\caption{\textbf{Performance on math reasoning benchmarks.} Our method, shown as a dashed red line, stands out and remains best or near-best on most tasks, especially on MATH-500, and it also achieves a strong overall performance on average.}
\label{fig:math_performance}
\end{figure}
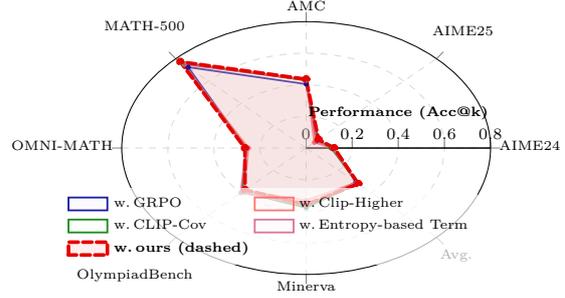

Despite its advantages, applying GRPO within reinforcement learning (RL) faces challenges in balancing exploration and exploitation (the E–E trade-off). \cite{sutton1988learning}. This challenge has driven a prevailing paradigm: existing methods either minimize entropy to yield convergent, credible outputs \cite{gao2025one,agarwal2025unreasonable},  or maximize it to boost exploration—such as DAPO’s decoupled clipping \cite{yu2025dapo}. Yet such entropy-maximizing or -minimizing approaches are inherently fragile: on the one hand, these blunt approaches to entropy manipulation lack stability; on the other hand, their inherent flaws without laborious parameter tuning, further manifest as policy entropy collapse \cite{yu2025dapo} and model collapse \cite{chu2025gpg, gao2025one, zheng2025group}. Specifically, a well-established relationship between model performance (R) and policy entropy (H) is defined as \(R = -a \cdot \exp(H) + b\) \cite{cui2025entropy}; this formulation reveals an inherent dilemma: improved performance comes at the cost of reduced entropy \cite{cheng2025reasoning}. To address this inherent trade-off, KL-divergence has become a prevalent choice. However, KL constraints alone can trigger model collapse in CoT reasoning \cite{chu2025gpg,gao2025one,zheng2025group}, further hindering stable policy learning. Moreover, other existing KL-based solutions—including ProRL’s KL regularization \cite{proRL2025} and newer Clip-Cov/KL-Cov methods \cite{cheng2025reasoning}—rely on intricate hyperparameters, which often require fine-tuning  dynamicsly tailored to token-level.

\textbf{Core Insight:}  We argue that these challenges, stemming from high-variance noise accumulation over extended reasoning sequences, are particularly acute in GRPO due to its critic-free nature \cite{zhang2025survey}.
When discovering a novel CoT structures, policies may undergo significant divergence from their initial distributions \cite{cheng2025reasoning}. 
These will incur the accumulation of high-variance noise throughout long CoT sequences, a phenomenon that is further exacerbated by GRPO \cite{zheng2025group}. Notably, entropy regularization alone may cause model collapse \cite{chu2025gpg,gao2025one,zheng2025group}, further destabilizing policy optimization. These issues are worsened in GRPO by its critic-free design \cite{zhang2025survey}, which lets high-variance noise accumulate in long reasoning sequences.

To address these issues, we propose \method, a novel framework that leverages Markov Likelihood \cite{chung1967markov} to adapt entropy control and connect group-level reward with token-level optimization. Our key contributions consist of three main components:
\begin{itemize}[leftmargin=1.1em]
\item \textbf{Markov Likelihood and Token Mean Optimization}: We use a Markov model to link group-level reward with tokens via token-level aggregation.
\item \textbf{Investigating Entropy Regularization in Critic-Free Paradigms}: We provide theoretical and empirical evidence that entropy regularization is not a good fit for the critic-free GRPO setting.
\item \textbf{Extensive Experiments}: Our method improves average accuracy by 2\% over baseline GRPO. We also run ablations on entropy regularization and sparse rewards, and the results support the effectiveness of \method.
\end{itemize}
Our experiments show that \method significantly outperforms existing methods on key performance metrics such as @k and accuracy, setting new records on mathematical reasoning tasks while maintaining superior training stability.

\section{Preliminaries}
This section lays the theoretical foundation for our method: we derive the policy-entropy gradient to motivate our regularizer, review policy-gradient basics and limits in GRPO.
\subsection{Policy Entropy Formulation}
In LLMs, the state s corresponds to the prompt context, and an action a refers to a token from the vocabulary \(\mathcal{A}\). For state is and action a, the policy is defined as follows:
\[
\pi_\theta(a \mid s) = \frac{\exp\!\left(\phi_\theta(s,a)\right)}{\sum_{a' \in \mathcal{A}} \exp\!\left(\phi_\theta(s,a')\right)},
\]
which is a softmax function.  Here, $\phi_\theta(s,a)\in\mathbb{R}$ is the score (the token logit), and $\theta=\{\phi_\theta(s,a)\mid s\in\mathcal{S},a\in\mathcal{A}\}$ collects all scores.
The entropy of policy distribution $\pi_\theta(\cdot \mid s)$ measures its uncertainty: \[ \mathcal{H}\left(\pi_\theta(\cdot \mid s)\right) = -\sum_{a} \pi_\theta(a \mid s) \log \pi_\theta(a \mid s). \] Applying the chain rule yields: 
{\small  
\begin{equation*}
\begin{split}
\frac{\partial \mathcal{H}}{\partial \phi_\theta(s, a_i)} &= -\sum_{a} \left[ \frac{\partial \pi_\theta(a \mid s)}{\partial \phi_\theta(s, a_i)} \log \pi_\theta(a \mid s) \right. \\ 
& \quad \left. + \pi_\theta(a \mid s) \frac{\partial \log \pi_\theta(a \mid s)}{\partial \phi_\theta(s, a_i)} \right].
\end{split}
\end{equation*}
} 
The partial derivatives are given by: \[ \frac{\partial \pi_\theta(a \mid s)}{\partial \phi_\theta(s, a_i)} = \begin{cases} \pi_\theta(a \mid s) \left(1 - \pi_\theta(a \mid s)\right), & a = a_i \\ -\pi_\theta(a \mid s) \pi_\theta(a_i \mid s), & a \neq a_i \end{cases}, \] \[ \frac{\partial \log \pi_\theta(a \mid s)}{\partial \phi_\theta(s, a_i)} = \begin{cases} 1 - \pi_\theta(a_i \mid s), & a = a_i \\ -\pi_\theta(a_i \mid s), & a \neq a_i \end{cases}. \]
Substituting these derivatives yields the following compact entropy-gradient expression: 
{\small
\begin{equation} \label{eq:entropy_gradient}  \frac{\partial \mathcal{H}}{\partial \phi_\theta(s, a_i)} = \pi_\theta(a_i \mid s) \left( \log \pi_\theta(a_i \mid s) + \mathcal{H}\left(\pi_\theta(\cdot \mid s)\right) \right). \end{equation}}

\subsection{Policy Gradient for LLMs Alignment}
We study RL fine-tuning for language models on verifiable tasks (RLVR) \cite{lambert2024tulu}, which aims to maximize a rule-based reward $A(\boldsymbol{y})=r(y)$ \cite{williams1992simple} :
\begin{equation}
\max_{\theta} J(\theta) := \mathbb{E}_{\boldsymbol{x} \sim \mathcal{D}, \boldsymbol{y} \sim \pi_\theta(\boldsymbol{x})} \left[ A(\boldsymbol{y}) \right],
\end{equation}
where $\boldsymbol{x} \sim \mathcal{D}$ is the input prompt, and $\boldsymbol{y} = \{y_1, \dots, y_T\}$ is the generated sequence of length $T$. Proximal Policy Optimization (PPO) \cite{schulman2017proximal} improves training stability by using a clipped important sampling:
{\small
\begin{equation} \label{eq:ppo_objective}
\begin{split}
L(\theta) &= \mathbb{E}_t \bigg[ \min\bigg(  \frac{\pi_\theta(y_t \mid \boldsymbol{y}_{<t})}{\pi_{\theta_{\text{old}}}(y_t \mid \boldsymbol{y}_{<t})} A_t, \\
& \operatorname{clip}\left( \frac{\pi_\theta(y_t \mid \boldsymbol{y}_{<t})}{\pi_{\theta_{\text{old}}}(y_t \mid \boldsymbol{y}_{<t})}, 1 - \epsilon, 1 + \epsilon \right) A_t \bigg) \bigg],
\end{split}
\end{equation}}
with $\epsilon$ controls the clipping range. The policy gradient components are derived as:
\[
\frac{\partial J}{\partial \phi_\theta(s, a_i)} = \mathbb{E}_{\pi_\theta} \left[ \frac{\partial \log \pi_\theta(a \mid s)}{\partial \phi_\theta(s, a_i)} A(s, a) \right].
\]

Plugging in the softmax gradient gives the basic policy gradient:
{\small
\begin{equation} 
\label{eq:policy_gradient}
\frac{\partial J}{\partial \phi_\theta(s, a_i)} = \pi_\theta(a_i \mid s) \left( A(s, a_i) - \mathbb{E}_{\pi_\theta} \left[ A(s, a) \right] \right),
\end{equation}}
where $\mathbb{E}_{\pi_\theta} \left[ A(s, a) \right] = \sum_{a} \pi_\theta(a \mid s) A(s, a)$.

\subsection{Limitations of Current Approaches}

Methods like RLOO \cite{kool2019buy} and GRPO \cite{shao2024deepseekmath} apply group-level normalization to cut variance while keeping training stable. We sample $K$ responses per prompt and compute the advantage as follows:
\begin{equation}
A_t = \frac{r\left( \boldsymbol{y} \right) - \operatorname{mean}\left( r\left( \boldsymbol{y}^{1:K} \right) \right)}{\operatorname{std}\left( r\left( \boldsymbol{y}^{1:K} \right) \right)}.
\end{equation}
However, these methods still struggle to balance exploration and exploitation, because group-level advantages are not suitable for sentence-level and token-level scenarios.

\begin{figure}[tb]
    \centering  
    \includegraphics[width=1\linewidth]{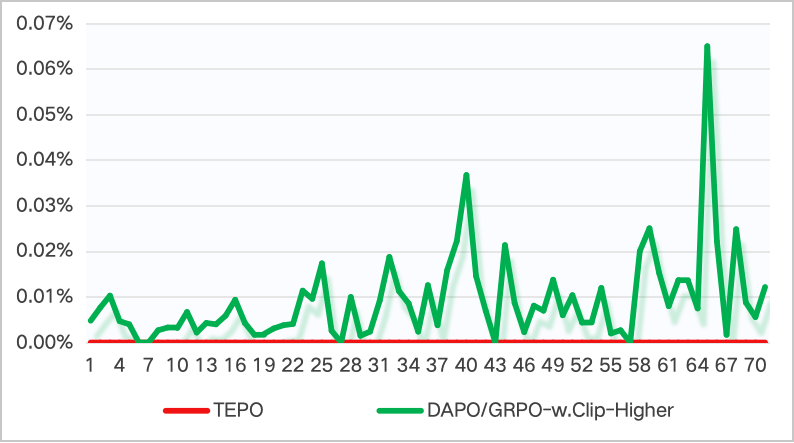}
    \caption{Clip Ratio over Steps}
    \label{fig:clip-ratio}
\end{figure}

\section{Methodology}\label{sec:method}

This section presents our proposed method. We first analyze the limitations of GRPO in entropy utilization, then derive our theoretical framework, and finally present the formal objective function and algorithm.

\subsection{Limitations of Entropy Regularization in GRPO}
\subsubsection{Policy Ascent Formulation}
At step $k$, starting from the current policy $\pi_k$, we get the next policy $\pi_{k+1}$ by solving the following constrained problem for each state $s$:
\[
\max_p \  \mathbb{E}_{a \sim p} \left[ A(s, a) \right] - \eta \cdot \text{KL} \left( p \| \pi_k(\cdot \mid s) \right)
\]
subject to $\sum_{a} p(a \mid s) = 1$. Using a Lagrangian view, these policy update methods lead to an iterative update rule written as:
{\small
\begin{equation} \label{eq:policy_update}
\pi_{k+1}(a \mid s) = \frac{\pi_k(a \mid s) \exp\left(\eta A^\pi(s, a)\right)}
{\mathbb{E}_{a' \sim \pi_k(\cdot \mid s)} \left[ \exp\left(\eta A^\pi(s, a')\right) \right]}.
\end{equation}}
More concretely, the normalized update is:
\[
\pi_{k+1}(a \mid s) \propto \pi_k(a \mid s) \exp\left(\frac{1}{\eta} A^\pi(s, a)\right).
\]
The update keeps the stability benefits of KL-regularized policy optimization while maximizing expected return:
\begin{equation} \label{eq:param_update}
\theta_{s,a}^{k+1} = \theta_{s,a}^k + \frac{1}{\eta} A^\pi(s, a).
\end{equation}
This \(\eta\) stabilizes the training process while counteracting policy ascent\cite{williams1992simple,kakade2001natural,schulman2015trust,schulman2017proximal,guo2025deepseek}.
Many studies show that strong policy learning needs sufficient exploration. When trying new CoT patterns, the policy can drift far from its starting distribution. In fact, using only a KL constraint can cause model collapse \cite{chu2025gpg,gao2025one,zheng2025group}, making stable training harder.

\subsubsection{Entropy-Policy Gradient Misalignment}
\begin{lemma}
For a softmax policy $\pi_\theta(a \mid s) \propto \exp(\phi_\theta(s,a))$:
\begin{itemize}[leftmargin=1.1em]
    \item When $A(s,a) < 0$ (suboptimal actions): $\langle \nabla_{\phi_\theta} \mathcal{H}, \nabla_{\phi_\theta} J \rangle > 0$
    \item When $A(s,a) > 0$ (optimal actions): $\langle \nabla_{\phi_\theta} \mathcal{H}, \nabla_{\phi_\theta} J \rangle < 0$
\end{itemize}
\end{lemma}
\begin{proof}
We analyze the inner product between entropy gradient and policy gradient:
\[
\langle \nabla_{\phi_\theta} \mathcal{H}, \nabla_{\phi_\theta} J \rangle = \sum_{a_i} \frac{\partial \mathcal{H}}{\partial \phi_\theta(s, a_i)} \cdot \frac{\partial J}{\partial \phi_\theta(s, a_i)}.
\]
Plugging in the entropy gradient (Eq.~\ref{eq:entropy_gradient}) and the policy gradient (Eq.  ~\ref{eq:policy_gradient}) gives:
{\small
\begin{equation*}
\begin{split}
&\langle \nabla_{\phi_\theta} \mathcal{H}, \nabla_{\phi_\theta} J \rangle \\
&= \sum_{a_i} \pi_\theta(a_i \vert s)^2 \left( \log \pi_\theta(a_i \vert s) + \mathcal{H}(\pi_\theta(\cdot \vert s)) \right) A(s, a_i).
\end{split}
\end{equation*}}

\paragraph{Case 1: Suboptimal Actions.} If the advantage is negative ($A(s,a)<0$):
\begin{itemize}[leftmargin=1.1em]
    \item Small probability $\pi_\theta(a \mid s) \to 0$ implies $\log \pi_\theta(a \mid s) \to -\infty$
    \item The term $\log \pi_\theta(a \mid s) + \mathcal{H}(\pi_\theta(\cdot \mid s))$ becomes negative
    \item With $\pi_\theta(a \mid s)^2 > 0$, $A(s,a) < 0$, and the logarithmic term negative, each product term is positive
\end{itemize}
Thus: $\langle \nabla_{\phi_\theta} \mathcal{H}, \nabla_{\phi_\theta} J \rangle > 0$.

\paragraph{Case 2: Optimal Actions.} If the advantage is positive ($A(s,a)>0$):
\begin{itemize}[leftmargin=1.1em]
    \item Large probability $\pi_\theta(a \mid s) \to 1^-$ implies $\log \pi_\theta(a \mid s) \to 0^-$
    \item The term $\log \pi_\theta(a \mid s) + \mathcal{H}(\pi_\theta(\cdot \mid s))$ remains negative
    \item With $\pi_\theta(a \mid s)^2 > 0$, $A(s,a) > 0$, and the logarithmic term negative, each product term is negative
\end{itemize}
Thus: $\langle \nabla_{\phi_\theta} \mathcal{H}, \nabla_{\phi_\theta} J \rangle < 0$.
\end{proof}
\subsubsection{Entropy-Natural Policy Gradient Misalignment}

\begin{lemma}
The entropy difference $\Delta \mathcal{H}$ correlates with the gradient inner product:
\begin{itemize}[leftmargin=1.1em]
\item When $A(s,a) < 0$(suboptimal actions), $\langle \nabla_{\phi_\theta} \mathcal{H}, \nabla_{\phi_\theta} J \rangle > 0$ and $\Delta \mathcal{H} > 0$ (entropy increases).
\item When $A(s,a) > 0$ (optimal actions),  $\langle \nabla_{\phi_\theta} \mathcal{H}, \nabla_{\phi_\theta} J \rangle < 0$ and $\Delta \mathcal{H} < 0$ (entropy decreases).
\end{itemize}
\end{lemma}
\begin{proof}
The policy parameter update follows:
\[
\phi_\theta(s, a) \leftarrow \phi_\theta(s, a) + \alpha \cdot \frac{\partial J}{\partial \phi_\theta(s, a)},
\]
where $\alpha > 0$ is the learning rate.

Using first-order Taylor expansion \cite{nocedal2006numerical}, the entropy change approximates as follows:
\[
\Delta \mathcal{H} \approx \sum_{a} \frac{\partial \mathcal{H}}{\partial \phi_\theta(s, a)} \cdot \Delta \phi_\theta(s, a).
\]

Plugging in the update rule $\Delta \phi_\theta(s, a) = \alpha \cdot \frac{\partial J}{\partial \phi_\theta(s, a)}$ gives:
\[
\Delta \mathcal{H} \approx \alpha \cdot \langle \nabla_{\phi_\theta} \mathcal{H}, \nabla_{\phi_\theta} J \rangle.
\]

When $A(s,a) < 0$(suboptimal actions), $\langle \nabla_{\phi_\theta} \mathcal{H}, \nabla_{\phi_\theta} J \rangle > 0$ and $\Delta \mathcal{H} > 0$.

When $A(s,a) > 0$ (optimal actions),  $\langle \nabla_{\phi_\theta} \mathcal{H}, \nabla_{\phi_\theta} J \rangle < 0$ and $\Delta \mathcal{H} < 0$.

This shows that, in no-critic setups, entropy regularization conflicts with policy ascent as Eq.  \ref{eq:policy_gradient}.  In no-critic setups, adding entropy control can hinder the improvement of final task performance. For example, maximizing entropy requires extensive hyperparameter tuning, and KL-Divergence can cause model collapse \cite{zheng2025group}.
\end{proof}

\subsection{Theoretical Analysis of \method}
\subsubsection{E-E Trade-off Analysis}
Policy Gradient (PG) \cite{williams1992simple} and Natural Policy Gradient (NPG) \cite{kakade2001natural}  keep updates stable and consistent with the KL-Divergence \cite{schulman2015trust}:
{\small
\begin{equation} \label{eq:entropy_covariance}
\begin{split}
&\mathcal{H}(\pi_\theta^{k+1} \mid s) - \mathcal{H}(\pi_\theta^k \mid s) \\
&\approx -\frac{1}{\eta} \cdot \operatorname{Cov}_{a \sim \pi_\theta^k(\cdot \mid s)} \left( \log \pi_\theta^k(a \mid s), r(s, a) \right),
\end{split}
\end{equation}}
where the nonnegative covariance term $\operatorname{Cov}$ tracks the policy’s step-by-step change. 

This links small entropy changes to the policy’s covariance structure across updates and articulates how entropy shifts relate to the exploitation policy. \textbf{This relationship thus reveals the core of exploration and exploitation: balancing the equation’s two sides achieves their optimal trade-off.}

In GRPO \cite{shao2024deepseekmath}, we study how policy entropy changes from iteration k to k+1. Given the state distribution (\(d^{\pi_k}\)) induced by \(\pi_k\), this entropy change decomposes as:
{\small
\begin{equation} \label{eq:entropy_decomposition}
\begin{split}
&\mathcal{H}(\pi_{k+1}) - \mathcal{H}(\pi_k) \\
&= \underbrace{\mathbb{E}_{s \sim d^{k+1}} \mathcal{H}(\pi_{k+1}(\cdot \mid s)) - \mathbb{E}_{s \sim d^k} \mathcal{H}(\pi_{k+1}(\cdot \mid s))}_{\text{State Distribution Shift}} \\
&\quad + \underbrace{\mathbb{E}_{s \sim d^k} \mathcal{H}(\pi_{k+1}(\cdot \mid s)) - \mathbb{E}_{s \sim d^k} \mathcal{H}(\pi_k(\cdot \mid s))}_{\text{Policy Update Effect}} \\
&\approx \mathbb{E}_{s \sim d^k} \left[ \mathcal{H}(\pi_{k+1}(\cdot \mid s)) - \mathcal{H}(\pi_k(\cdot \mid s)) \right],
\end{split}
\end{equation}}
where the approximation Eq. \ref{eq:entropy_covariance} holds under the assumption $d^{\pi_{k+1}} \approx d^{\pi_k}$. 

However, the required near-stationarity (\(d^{\pi_{k+1}} \approx d^{\pi_k}\)) is hard to achieve in CoT reasoning. GRPO uses a no-critic setup. This makes the problem worse. Policy updates can change the policy greatly. This leads to large distribution shifts. And it breaks the approximation. 

Critic-free GRPO is ill-suited for entropy regularization. Mainly because KL is always non-negative, so it fails to capture small, useful reward differences when signals exist only in a few tokens. Increasing entropy leads to more wrong trials and long-term degradation, while decreasing it rapidly erodes diversity and crashes the model.

\subsubsection{Markov Likelihood Resolves the E-E Trade-Off}
In GRPO, the token-level importance weights $\frac{\pi_\theta(y_{i,t} \mid x, y_{i,<t})}{\pi_{\theta_{\text{old}}}(y_{i,t} \mid x, y_{i,<t})}$ do not fully reflect how the whole sequence distribution changes. This leads to very noisy (high-variance) gradient estimates and can often cause model collapse that is hard to undo \cite{zheng2025group,chen2025minimax,chu2025gpg}. The main reason is that GRPO has no critic (value baseline) and uses only sparse token-level rewards.

LLMs generate text one token at a time: each token’s probability depends on prior tokens (a Markov factorization) \cite{chung1967markov}. Thus, the sequence-level (Markov Likelihood) importance ratio \(\mathcal{IS}_i(\theta)\) factors into a product of token ratios:
{\small
\begin{equation*}
\begin{split}
&\mathcal{IS}_i(\theta)=\left( \frac{\pi_\theta(y_i \mid x)}{\pi_{\theta_{\text{old}}}(y_i \mid x)} \right)^{\frac{1}{|y_i|}} \\
&= \exp\left( \frac{1}{|y_i|} \sum_{t=1}^{|y_i|}
\log \frac{\pi_\theta(y_{i,t} \mid x, y_{i,<t})}{\pi_{\theta_{\text{old}}}(y_{i,t} \mid x, y_{i,<t})} \right),
\end{split}
\end{equation*}}
which represents a geometric mean.

\subsection{Proposed Method}

\subsubsection{Objective Function}

Putting the update rule (Eq. ~\ref{eq:param_update}) and the entropy–covariance link (Eq. ~\ref{eq:entropy_covariance}) into the PPO loss gives the following joint objective:
{\small
\begin{equation} \label{eq:ppo_objective_modified}
\begin{split}
L(\theta) &= \frac{1}{\sum_{i=1}^G |o_i|} \sum_{i=1}^G \sum_{t=1}^{|o_i|} \min\Bigg( \mathcal{IS}_i(\theta) \hat{A}_{i,t}, \\
&\quad \text{clip}\left( \mathcal{IS}_i(\theta),\ 1 - \varepsilon,\ 1 + \varepsilon \right) \hat{A}_{i,t} \Bigg), \\ 
&\text{s.t.}\ 0 < \left| \left\{ o_i \mid \text{is\_equivalent}(a, o_i) \right\} \right| < G,
\end{split}
\end{equation}}
where the $\mathcal{IS}_i(\theta)$ operates in the backward pass as: 
$\frac{\partial L}{\partial \pi_\theta} = \frac{\partial L}{\partial \mathcal{IS}_i} \cdot \mathcal{IS}_i \cdot \frac{\text{mask}_{i,t}}{|o_i| \cdot \pi_\theta}$.

\subsubsection{Computation Graph for the Token Level}
We use a careful backward pass to keep training stable and consistent with our theory, mainly by handling importance-sampling ratios at both the sequence and token levels.
\begin{algorithm}[H]
\small
\caption{Simplified Backward Iteration for Policy Gradient}
\label{alg:backward_acl}
\begin{algorithmic}[1]
\Require{
  Loss function: $L(\theta) = \frac{1}{\text{total\_mask}} \sum_{i=1}^G \sum_{t=1}^{|o_i|} \mathcal{IS}_i(\theta) \hat{A}_{i,t} \cdot \text{mask}_{i,t}$
}
\Ensure{
  Policy parameter gradient $\nabla_\theta L(\theta)$
}
\State Compute gradient of $L$ w.r.t. total loss sum ($\text{sum\_loss} = \sum_{i,t} \mathcal{IS}_i \hat{A}_{i,t} \cdot \text{mask}_{i,t}$):
\State \quad $\frac{\partial L}{\partial \text{sum\_loss}} = \frac{1}{\text{total\_mask}}$
\State Propagate gradient to token-level term $\mathcal{IS}_i \hat{A}_{i,t}$:
\State \quad $\frac{\partial L}{\partial (\mathcal{IS}_i \hat{A}_{i,t})} = \frac{\text{mask}_{i,t}}{\text{total\_mask}}$
\State Calculate gradient of importance sampling ratio $\mathcal{IS}_i$:
\State \quad $\frac{\partial L}{\partial \mathcal{IS}_i} = \sum_{t=1}^{|o_i|} \hat{A}_{i,t} \cdot \frac{\text{mask}_{i,t}}{\text{total\_mask}}$
\State Backpropagate to current policy probability $\pi_\theta$:
\State \quad $\frac{\partial L}{\partial \pi_\theta} = \frac{\partial L}{\partial \mathcal{IS}_i} \cdot \mathcal{IS}_i \cdot \frac{\text{mask}_{i,t}}{|o_i| \cdot \pi_\theta}$
\State Obtain $\nabla_\theta L(\theta)$ by backpropagating through the policy network.
\end{algorithmic}
\end{algorithm}
The algorithm follows the gradient in Eq. ~\ref{eq:ppo_objective_modified}. Here, $\frac{\partial L}{\partial \mathcal{IS}_i} \cdot \mathcal{IS}_i$ gives a batch-level signal for each sequence, while $\frac{\text{mask}_{i,t}}{|o_i| \cdot \pi_\theta}$ gives a sequence-normalized signal for each token. Using the geometric mean in $\mathcal{IS}_i(\theta)$ keeps gradients stable across different sequence lengths and reduces the variance spikes noted in Sec.~\ref{sec:method}.

\subsubsection{\method with Empirical Proof}
Our method addresses these limitations through the following steps:
\begin{itemize}[leftmargin=1.1em]
\item Using a Markov Likelihood to connect group-level reward with tokens, which reduces gradient bias (see Fig.~\ref{fig:clip-ratio}).
\item Adopting a token-leveln calculation to enable token-level policy optimization.
\item Maintaining training stability while preserving the ability to explore (see Fig.~\ref{fig:tepo_vs_grpo}).\end{itemize}
These results match the known link between performance and policy entropy: ($R = -a \cdot \exp(H) + b$) \cite{cui2025entropy}. As it notes, optimizing downstream tasks tends to lower entropy; pushing entropy higher rarely helps, and adding a KL term often hurts. The reason is that inCoT, token distributions shift across steps, so entropy and KL affect tokens unevenly and lead to collapse.

\section{Experiment}
\sisetup{
  table-number-alignment = center,  
  table-format = 2.2,            
  round-mode = places,          
  round-precision = 2,            
  detect-weight = true,     
  detect-shape = true
}

\begin{table*}[t]
\centering
\caption{Performance Comparison of Qwen2.5-7B on Mathematical Reasoning Benchmarks (Accuracy Percentage \%). Best results are in \textbf{bold}, second best are \underline{underlined}.}
\setlength{\tabcolsep}{4pt}
\renewcommand{\arraystretch}{1.15}
\resizebox{\textwidth}{!}{
  \begin{tabular}{l
                  S S S S S S S S}
    \toprule
    \textbf{Method} & 
    \textbf{AIME24} & \textbf{AIME25} & \textbf{AMC} & \textbf{MATH-500} &
    \textbf{OMNI-MATH} & \textbf{OlympiadBench} & \textbf{Minerva} & \textbf{Avg.} \\ 
    \midrule
    \textbf{Qwen2.5-7B}    & 0.9375 & 0.9375 & 14.3400 & 43.2000 & 13.7300 & 16.7400 & 21.6900 & 13.3000 \\
    w.\ GRPO               & 11.5600 & 6.2500 & 40.4700 & 72.4000 & 26.0000 & 37.7700 & \underline{\num{36.7600}} & 30.3000 \\
    w.\ Clip-Higher        & 11.2500 & 5.0000 & 42.8000 & \underline{\num{76.6000}} & 25.0000 & 37.9200 & 33.0800 & 30.8500 \\
    w.\ CLIP-Cov           & 10.8300 & 7.4000 & 42.8400 & 75.6000 & 26.1100 & \underline{\num{39.4000}} & \bfseries 37.8600 & \underline{\num{31.6300}} \\
    w.\ Entropy-based Term & \underline{\num{11.3500}} & \underline{\num{6.1500}} & \underline{\num{43.1800}} & 74.8000 & \underline{\num{26.1400}} & \bfseries 40.1400 & 36.0200 & 31.6200 \\
    \rowcolor{red!6}\textbf{\method}       & \bfseries 12.1800 & \bfseries 7.6040 & \bfseries 43.5600 & \bfseries 77.2000 & \bfseries 26.5700 & 37.4800 & 35.6600 & \bfseries 32.0400 \\
    \bottomrule
  \end{tabular}%
}
\label{tab:math_benchmark_performance}
\end{table*}
\begin{figure*}[tb]
\centering
\setlength{\parindent}{0pt}
\begin{subfigure}{0.45\textwidth}
\centering
\includegraphics[width=\linewidth]{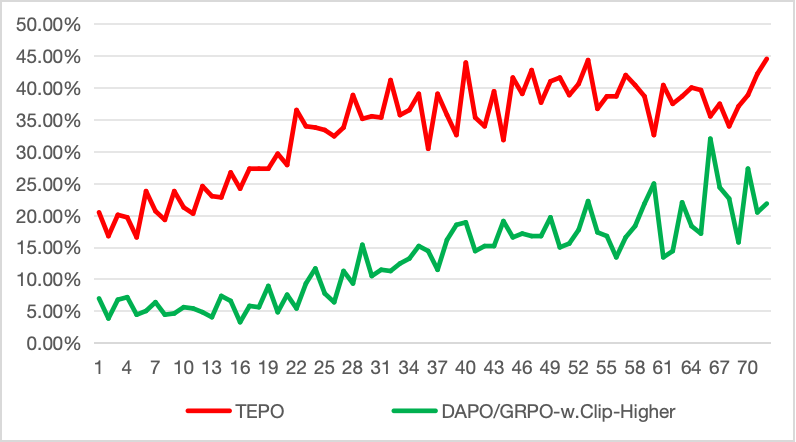}
\subcaption{Reward progression across training steps: TEPO shows more stable and consistently higher rewards than GRPO, which indicates better optimization.}
\label{fig:tepo_grpo_reward}
\end{subfigure}
\hspace{0.05\textwidth}
\begin{subfigure}{0.45\textwidth}
\centering
\includegraphics[width=\linewidth]{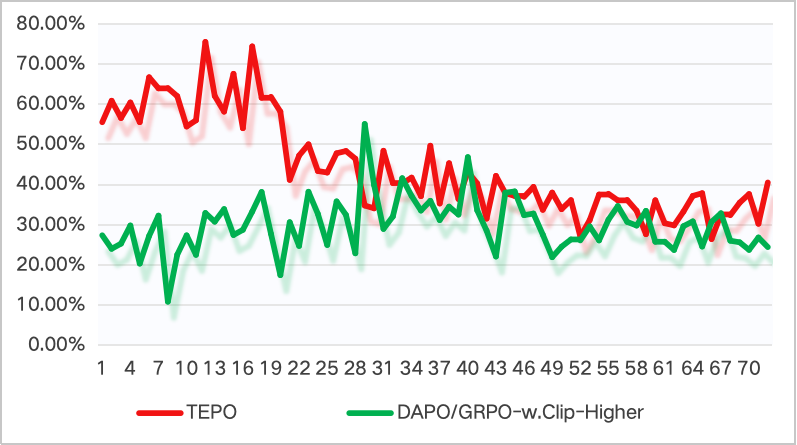}
\subcaption{Gradient Norm over training steps: TEPO maintains sustain higher gradient norms, suggesting more active policy updates and deeper reasoning.}
\label{fig:tepo_grpo_grad}
\end{subfigure}

\caption{Comparative analysis of TEPO versus DAPO/GRPO w. Clip-Higher training dynamics. The left panel shows reward progression, while the right panel displays gradient norm throughout training.}
\label{fig:tepo_vs_grpo}
\end{figure*}
\subsection{Experimental Setup}
\paragraph{Implementation Details.}
We ran all experiments on eight NVIDIA A100 GPUs. For each rollout step, we processed 64 prompts per batch and sampled 8 responses per prompt with temperature 1.0. The policy was updated 8 times using these responses. To keep training effective, we removed prompts whose sampled responses were all correct or all wrong, following \cite{yu2025dapo}. Key hyperparameters were: learning rate $lr = 5 \times 10^{-7}$, maximum prompt length $max\_prompt\_length = 2024$, maximum response length $max\_response\_length = 8192$, and training prompt mini-batch size $train\_prompt\_mini\_bsz = 16$. 
\paragraph{Datasets.}
We trained Qwen2.5 models \cite{yang2025qwen2} on DAPO-MATH \cite{yu2025dapo} and evaluated on seven math benchmarks: MATH-500, AIME24/25 \cite{li2024numinamath}, AMC, OMNI-MATH, OlympiadBench, and Minerva \cite{lewkowycz2022solving}. For AIME and AMC we used temperature 0.6; for the others we used greedy decoding. The maximum generation length was 8192 tokens. AIME, AIME25, and AMC results were reported with 32 samples per problem (@32) following prior work \cite{guo2025deepseek,guo2025deepseekr1,shao2024deepseekmath}.
\paragraph{Baseline Methods.}
We compared \method with strong baselines. Original GRPO/DAPO \cite{shao2024deepseekmath} used the standard formulation without extra changes. GRPO/DAPO with Clip-Higher applied an upper threshold $\epsilon = 0.28$ in the PPO loss \cite{yu2025dapo}. Clip-Cov \cite{cui2025entropy} clipped tokens with high covariance (clip ratio $r = 2 \times 10^{-4}$). An entropy-based term \cite{cheng2025reasoning} added entropy regularization to advantage estimation with scale $\alpha = 4 \times 10^{-4}$ and clipping parameter $\kappa = 2$. For \method, we set $\epsilon = 0.2$.

\subsection{Main Results}
Table \ref{tab:math_benchmark_performance} reports results on seven math benchmarks. \method attains the best average accuracy at 32.04\%, beating all baselines. Compared with GRPO, the average gain is 1.74 points (32.04\% vs.\ 30.30\%). The largest gain appears on MATH-500: \method reaches 77.20\%, a +4.8 point jump over GRPO. This suggests \method handles sparse rewards better on hard math tasks. On AIME, it stays competitive, scoring 12.18\% on AIME24 and 7.60\% on AIME25.
Our method performs relatively poorly on the Minerva and OlympiadBench datasets. The core reason is Minerva and OlympiadBench’s strict output formatting, complex answers (unlike other datasets’ 0-10000 numerical outputs), and especially LaTeX requirement. These two datasets hurt exploration-oriented methods’ performance, including DAPO/GRPO w.Clip-Higher.
\begin{figure}[t]
    \centering
    \includegraphics[width=1\linewidth]{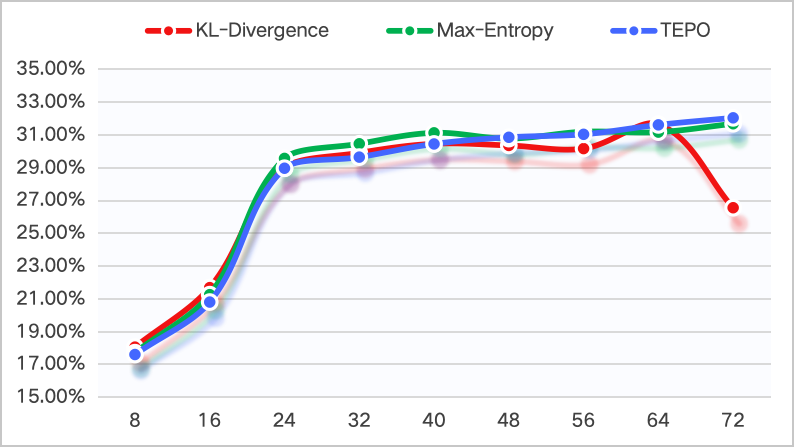}
    \caption{Training dynamics with different entropy regularization strategies. The complete \method demonstrates superior stability and performance compared to variants with maximum entropy or KL-divergence regularization.}
    \label{fig:entropy-regularity}
\end{figure}
\begin{figure}[t]
    \centering
    \includegraphics[width=1\linewidth]{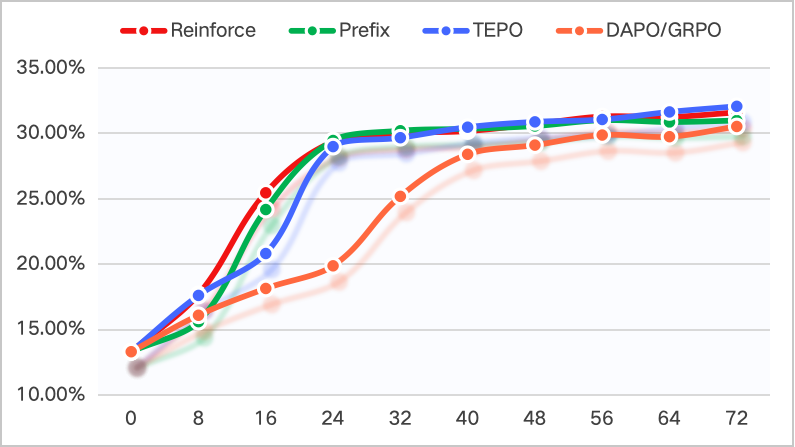}
    \caption{Comparison of importance sampling strategies in sparse reward environments. Our approach demonstrates more effective utilization of sparse learning signals compared to REINFORCE and prefix-based importance sampling.}
    \label{fig:prefix}
\end{figure}
Figure \ref{fig:tepo_vs_grpo} presents the training curves. In Fig.\ \ref{fig:tepo_grpo_reward}, \method achieves steadier and higher rewards over steps, demonstrating more effective learning. In Fig.\ \ref{fig:tepo_grpo_grad}, \method’s gradient norms are consistently higher, reflecting more active parameter updates. These trends validate our assertion that token-level explicit policy optimization delivers finer-grained learning signals.
\sisetup{
  table-number-alignment = center,
  table-format = 2.2,        
  round-mode = places,    
  round-precision = 2,       
  detect-weight = true
}
\begin{table*}[t]
\centering
\caption{Ablation on entropy regularization and importance sampling (Accuracy, \%). Higher is better.}
\setlength{\tabcolsep}{4pt}
\renewcommand{\arraystretch}{1.15}
\resizebox{\textwidth}{!}{%
\begin{tabular}{l
                S S S S S S S S}
\toprule
\textbf{Method} &
\textbf{AIME24} & \textbf{AIME25} & \textbf{AMC} & \textbf{MATH-500} &
\textbf{OMNI-MATH} & \textbf{OlympiadBench} & \textbf{Minerva} & \textbf{Avg.} \\
\midrule
\multicolumn{9}{l}{\textit{Panel A: Entropy Regularization}} \\
\addlinespace[2pt]
w.\ Max-Entropy   & 13.0200 & 5.9370 & 42.8000 & 76.6000 & 26.8500 & 38.3700 & 37.1300 & 31.6500 \\
w.\ KL-Divergence &  5.5200 & 3.9580 & 42.8400 & 75.6000 & 22.4100 & 32.7400 & 30.8800 & 26.2500 \\
\rowcolor{red!6}\textbf{\method}
                  & \bfseries 12.1800 & \bfseries 7.6040 & \bfseries 43.5600 & \bfseries 77.2000 &
                    \bfseries 26.5700 & \bfseries 37.4800 & \bfseries 35.6600 & \bfseries 32.0400 \\
\midrule
\multicolumn{9}{l}{\textit{Panel B: Importance Sampling}} \\
\addlinespace[2pt]
w.GRPO        & 11.5600 & 6.2500 & 40.4700 & 72.4000 & 26.0000 & 37.7700 & 36.7600 & 30.3000 \\
w.\ Reinforce & 11.0400 & 6.7700 & 42.8000 & 74.8000 & 26.9600 & 38.5100 & 32.7200 & 31.5800 \\
w.\ Prefix IS & 11.0800 & 6.6660 & 42.4400 & 73.4000 & 25.1100 & 39.1100 & 35.6600 & 30.9500 \\
\rowcolor{red!6}\textbf{\method}
              & \bfseries 12.1800 & \bfseries 7.6040 & \bfseries 43.5600 & \bfseries 77.2000 &
                \bfseries 26.5700 & \bfseries 37.4800 & \bfseries 35.6600 & \bfseries 32.0400 \\
\bottomrule
\end{tabular}%
}
\label{tab:ablation_combined}
\end{table*}
\subsection{Ablation Studies}

\subsubsection{Analysis of Entropy Regularization Components}
 We ran ablations to study the entropy parts of our method. Table \ref{tab:ablation_combined} shows the numbers, and Figure \ref{fig:entropy-regularity} shows the training curves.
 
Both maximum-entropy and KL-based regularization hurt performance in our GRPO setup. The maximum-entropy variant lowers the average accuracy by about 0.39 points (31.65\% vs.\ 32.04\%). The KL variant is worse, dropping accuracy by 5.79 points (26.25\% vs.\ 32.04\%).

We think this happens because these regularizers work best when rewards are dense, but math reasoning gives sparse feedback. In GRPO, only a small set of token actions get useful signals. Adding entropy terms spreads probability mass and weakens those scarce signals. The stronger drop with KL suggests that forcing the policy to stay close to the initial model makes it hard to move toward the few actions that earn reward, so the model adapts less to the sparse feedback.
\subsubsection{Analysis of Sparse Reward in GRPO}
We further evaluated different importance sampling designs for sparse rewards. Table \ref{tab:ablation_combined} presents the scores, and Figure \ref{fig:prefix} illustrates performance across steps.

The REINFORCE importance sampling employs $sg[\mathcal{IS}_i(\theta)] \cdot \log\pi$, implementing an online single-step policy similar to bandit sentence sampling. While this approach shows some improvement over basic GRPO, it still underperforms our method. The prefix importance sampling strategy computes:
\vspace{-1em}
{\small
\begin{equation*}
  \mathcal{IS}_{i,t}(\theta) = \left( \frac{\pi_\theta(y_{i,j \leq t} \mid x)}{\pi_{\theta_{\text{old}}}(y_{i,j \leq t} \mid x)} \right)^{\frac{1}{|y_{i,j \leq t}|}},
\end{equation*}}
where $y_{i,j \leq t}$ represents the prefix of the $i$-th sentence up to the $t$-th token.

Our method is best among the three. It is +1.09 points over prefix importance sampling (32.04\% vs.\ 30.95\%) and +0.46 points over REINFORCE-style sampling (32.04\% vs.\ 31.58\%). Relative to plain GRPO, it improves the average by 1.74 points.

Prefix importance sampling shows the critic-free GRPO is a sparse rewards method. Importance sampling in REINFORCE reveals bias in sentence-level rewards. Ours balances group-level credit with token-level updates via a sentence likelihood, which reduces variance and performs better under sparse rewards.

\section{Related work}

\subsection{Exploration-Exploitation Trade-off in Traditional RL} 

Balancing exploration and exploitation is a key problem in RL \cite{sutton1998reinforcement}. Under the maximum-entropy view, an entropy term gives a clear way to handle this trade-off \cite{ziebart2008maximum, toussaint2009robot, chao2024maxentflow}. This idea shows up in common algorithms: DQN usually relies on $\epsilon$-greedy exploration instead of an explicit entropy term \cite{mnih2015human}, PPO often adds an entropy bonus to keep exploration active \cite{schulman2017proximal}, and SAC directly optimizes a maximum-entropy objective \cite{haarnoja2017reinforcement, haarnoja2018soft}. In RL with LLMs, the role of entropy is still unclear. RLHF training typically uses a KL penalty to a reference policy \cite{ouyang2022training, hu2024openrlhf}. GRPO-style objectives and recent studies often find little or unclear benefit from standard entropy bonuses, so the effect on generation quality and training stability remains an open question \cite{shao2024deepseekmath, guo2025deepseekr1, he2025skyworkor1, cui2025entropylmrl, shen2025entropycontrol, zhang2024erprm, rafailov2023dpo, hong2024orpo, xiao2024dposurvey}.
\subsection{RL Advancements in LLMs}
RL is a key part of LLM post-training \cite{ouyang2022training, grattafiori2024llama}, and combining it with verifiable rewards further improves reasoning \cite{jaech2024openai, guo2025deepseek}. Recent work shows three clear trends. First, simple policy-gradient methods still work well: GPG keeps training straightforward by using REINFORCE \cite{chu2025gpg}, and CISPO improves efficiency by clipping importance weights and detaching them from the update; both use a standard policy-gradient loss \cite{chen2025minimax}. Second, GSPO shifts to sequence-level learning: it defines importance ratios with whole-sequence likelihoods and applies clipping, rewards, and updates at the sequence level \cite{zheng2025group}. Third, entropy matters: \cite{cui2025entropy} find performance (R) fits \( R = -a \exp(\mathcal{H}) + b \), meaning lower entropy usually gives better results. This matches RLVR findings \cite{jaech2024openai, guo2025deepseek, lambert2024tulu} and aligns with LLM scaling laws \cite{kaplan2020scaling, hoffmann2022training, gao2023scaling}.

\section{Conclusion}
GRPO has improved LLM mathematical reasoning but shows fragile convergence in CoT: uniform token-level entropy often triggers collapse or early stop. We introduce \method, which links group loss to sentence representations via a Markov likelihood and then to tokens through token-mean aggregation, stabilizing updates. On math benchmarks, \method surpasses baselines (e.g., Acc@k), sets a new state of the art, and markedly improves training stability.

\section*{Limitations}
Although our method proposes a novel token-level framework that incorporates Markov Likelihood to link group-level rewards with tokens via token-mean aggregation, it only connects tokens to group-level rewards. It leaves token optimization to backward and fails to further distinguish how different tokens affect performance. Future research should focus on exploring tokens' roles in sentences and how to link with group-level reward via a general paradigm.

\bibliography{arxiv}

\newpage
\appendix

\section{Policy Optimization with KL-Divergence}
The objective of GRPO or RL in LLMs is to identify a new policy, \(\pi_\theta\), which is as closely aligned as possible with the ideal target distribution \(\pi^*\)\cite{williams1992simple, chu2025gpg, schulman2017proximal, schulman2015trust, kakade2001natural, shao2024deepseekmath}. This task is equivalent to minimizing the KL- divergence, expressed as:

\begin{equation}
\min_\theta \text{KL}(\pi^*(a|s) \Vert \pi_\theta(a|s))
\end{equation}
Substitute Eq. \ref{eq:policy_update}, we get :
\begin{align*}
&\text{KL}(\pi_\theta \Vert \pi^*) \\
&\propto \mathbb{E}_{a \sim \pi_\theta} \left[ \left( \log \pi_\theta(a|s) - \log \pi_{\theta_{\text{old}}}(a|s) \right) - \frac{\hat{A}(a, s)}{\beta} \right] \\
&= \frac{1}{\beta} \mathbb{E}_{a \sim \pi_\theta} \left[ \beta \text{KL}(\pi_\theta \Vert \pi_{\theta_{\text{old}}}) - \hat{A}(a, s) \right]
\end{align*}
Therefore, the PPO objective is essentially solving a distribution matching problem toward a target distribution shaped by the advantage function.
\section{Policy Entropy in LLMs} \label{appendix-kl}
To extend the state-specific entropy difference to the global entropy change (i.e., entropy averaged over all states, which better reflects the overall exploration-exploitation balance of the policy), we first rely on a first-order Taylor expansion of entropy around the policy parameters \( \theta^k \). For any state \( s \), the entropy of the updated policy \( \mathcal{H}(\theta^{k+1} \mid s) \) can be approximated as:  
\(\mathcal{H}(\theta^k \mid s) + \left\langle \nabla_\theta \mathcal{H}(\theta^k \mid s), \theta^{k+1} - \theta^k \right\rangle, \)
where \( \nabla_\theta \mathcal{H}(\theta^k \mid s) \) is the gradient of entropy with respect to parameters \( \theta^k \), and \( \langle \cdot, \cdot \rangle \) denotes the inner product. This approximation holds when the parameter update step \( \theta^{k+1} - \theta^k \) is small (a common assumption in policy optimization to ensure stability), as it ignores higher-order terms in the entropy’s parameter dependence. Additionally, we define \( d^{\pi_{k+1}} := \langle \nabla_\theta \mathcal{H}(\theta^k \mid s), \theta^{k+1} - \theta^k \rangle \), which quantifies the "contribution" of the entropy gradient to the state distribution update.  

 The entropy difference relies on a critical approximation:  
{\small
\begin{equation*}
\begin{split}
&\mathcal{H}(\pi_{k+1}) - \mathcal{H}(\pi_k) \\
&= \mathbb{E}_{s \sim d^{\pi_{k+1}}} \left[ \mathcal{H}(\pi_{k+1}(\cdot \mid s)) \right] - \mathbb{E}_{s \sim d^{\pi_k}} \left[ \mathcal{H}(\pi_k(\cdot \mid s)) \right] \\
&\approx \mathbb{E}_{s \sim d^{\pi_k}} \left[ \mathcal{H}(\pi_{k+1}(\cdot \mid s)) \right] - \mathbb{E}_{s \sim d^{\pi_k}} \left[ \mathcal{H}(\pi_k(\cdot \mid s)) \right] \\
&= \mathbb{E}_{s \sim d^{\pi_k}} \left[ \mathcal{H}(\pi_{k+1}(\cdot \mid s)) - \mathcal{H}(\pi_k(\cdot \mid s)) \right]\\
&=-1/\eta \cdot \operatorname{Cov}_{a \sim \pi_\theta^k(\cdot \mid s)} \left( \log \pi_\theta^k(a \mid s), A(s, a) \right)
\end{split}
\end{equation*}}  
Here, \( d^{\pi_t} \) (for \( t = k, k+1 \)) denotes the state distribution of the actor network at iteration \( t \)—specifically, the distribution of states encountered by the policy during interaction with the environment (e.g., \( d^{\pi_{k+1}} \) aligns with the updated policy \( \pi_{k+1} \)). This approximation is justified by three key considerations:  
\begin{itemize}[leftmargin=1.1em]
\item The policy update step is small, inducing high structural similarity and distributional overlap between \( d^{\pi_{k+1}} \) and \( d^{\pi_k} \)—a natural consequence of the step-wise policy update process \cite{schulman2015trust}.
\item Replacing \( d^{\pi_{k+1}} \) with \( d^{\pi_k} \) (from the first to the second line) incurs negligible error under first-order approximation: the visitation probability difference \( d^{\pi_{k+1}} - d^{\pi_k} \) is a second-order term relative to the policy vector error \( \pi_{k+1} - \pi_k \)  by Taylor Equation \cite{nocedal2006numerical}.  
\item Residual bias from distribution shifts is corrected via importance sampling \cite{schulman2017proximal}, preserving the approximation’s unbiasedness (i.e., the approximation error has zero expectation).
\end{itemize}

\section{Covariance Term}
To refine the covariance term (for scenarios like LLM sequence generation), we substitute two key components enhancing update stability and interpretability:  

1. Normalized Advantage Function  
Unlike immediate reward \( r(s,a) \), \( A(s,a) \) quantifies action \( a \)'s "relative value" under \( s \) (e.g., \( A(s,a) = r(s,a) - V(s) \), with \( V(s) \) as the state value function). This isolates \( a \)'s incremental benefit over other actions, making the covariance term a better policy update guide.  Normalized via \( K \) reward samples ( \( \boldsymbol{y}^{1:K} \) from \( \pi_k \))  like RLOO \cite{kool2019buy} and GRPO \cite{shao2024deepseekmath}  as:  
   \(A(s, a) = \frac{r(\boldsymbol{y}) - \operatorname{mean}(r(\boldsymbol{y}^{1:K}))}{\operatorname{std}(r(\boldsymbol{y}^{1:K}))},\)  
this scales \( A(s,a) \) to zero mean/unit variance, preventing large reward magnitudes from dominating updates and stabilizing optimization.
   
2. Token-Level Log-Probability Difference  
For sequence-generation tasks , we define the token-level log-probability difference to approximate the current policy’s expected log-probability over the state distribution \(d^{\pi_k}\):  
\begin{equation*}
\begin{split}
&\mathbb{E}_{s \sim d^{\pi_k}} \left[ \log \pi_\theta^k(a \mid s) \right] \\
&= \log \pi_\theta^{k}(a \mid s) - \mathbb{E}_{s \sim d^{\pi_k}}\left[\log \pi_\theta^k(a \mid s)\right] \\
\end{split}
\end{equation*}

as the parameter update \(\theta^{k+1} - \theta^k\) induces only linear changes in log-probabilities. The right-hand side matches the form of the Kullback-Leibler (KL) divergence between successive policies.

Substituting these components yields the step-wise covariance \(\operatorname{Cov}\):  
{\small
\begin{equation*} 
    \begin{split}
        \text{Cov}(y_i) &= \left( A(y_i) - \frac{1}{N} \sum_{j=1}^{N} A(y_j) \right) \\
        &\cdot \left( \log \pi_\theta(y_i) - \frac{1}{N} \sum_{j=1}^{N} \log \pi_\theta(y_j) \right)
    \end{split}
\end{equation*}}
This links token-level preference changes and normalized relative values. Optimizing it balances reward maximization and entropy (exploration)—critical for LLM generation, where excessive entropy causes incoherence and insufficient entropy leads to rigidity.

\section{Use of LLMs}
Large language models (LLMs), specifically GPT-5 and DeepSeek-R1, were used solely as a supplementary tool during the preparation of this work for tasks such as polishing the writing. The authors are solely responsible for the entire research conception, technical direction, scientific content, and interpretation of results. The LLMs were employed only to assist in the presentation and clarity of the manuscript.

\end{document}